\documentclass{article}
\usepackage{kr}

\usepackage{times}
\usepackage{soul}
\usepackage{url}
\usepackage[hidelinks]{hyperref}
\usepackage[utf8]{inputenc}
\usepackage[small]{caption}
\usepackage{graphicx}
\usepackage{amsmath}
\usepackage{amsthm}
\usepackage{booktabs}
\usepackage{algorithm}
\usepackage{algorithmic}
\usepackage{latexsym}

\usepackage{amssymb}
\usepackage{mathtools}
\usepackage{stmaryrd}
\usepackage{esvect}
\usepackage{braket}

\usepackage[dvipsnames]{xcolor}
\usepackage{xstring}
\usepackage{xspace}
\usepackage{graphicx}
\usepackage{boxedminipage}
\usepackage{multirow}
\usepackage{url}
\usepackage{enumerate}
\usepackage[normalem]{ulem}

\usepackage{subcaption}
\usepackage{tikz}

\usepackage{microtype}

\newtheorem{example}{Example}

\newtheorem{definition}{Definition}
\newtheorem{proposition}{Proposition}

\newcommand{\ol}[1]{\overline{#1}}
\newcommand{\cspMIof}[2]{\ensuremath{\mathit{CR}^{#2}(#1)}} %

\newcommand{\OCFsolutionsOf}[1]{\ensuremath{\mathit{Sol_{OCF}}(#1)}}

\newcommand{\OCFsolutionsRnMI}[1]{\ensuremath{\OCFsolutionsOf{\cspMIof{\R_n}{n-1}}}}

\newcommand{\syntheticKB}[1]{\texttt{kb\_synth<\(n\)>\_c<\(2n\!\!-\!\!1\)>.pl}}

\newcommand{\R}{\ensuremath{\mathcal R}}

\newcommand{\Mod}{\mbox{\it Mod}\,}

\makeatletter
\ifx\centernot\undefined
\newcommand*{\centernot}{%
	\mathpalette\@centernot
}
\fi
\def\@centernot#1#2{%
	\mathrel{%
		\rlap{%
			\settowidth\dimen@{$\m@th#1{#2}$}%
			\kern.5\dimen@
			\settowidth\dimen@{$\m@th#1=$}%
			\kern-.5\dimen@
			$\m@th#1\not$%
		}%
		{#2}%
	}%
}
\makeatother
\DeclareRobustCommand\nmableitSymb{\mathrel|\mkern-.5mu\joinrel\sim} %
\newcommand{\nmableit}{\ensuremath{\mbox{$\,\nmableitSymb\,$}}} %

\newcommand{\beweisendezeichen}%
{\penalty50\hspace*{0pt plus 1fil}\parfillskip=0pt\mbox{$\Box$}}

\newcommand{\fussnoteOhneMarkierung}[1]%
{%
\footnote{#1}%
\addtocounter{footnote}{-1}%
}

\newcommand{\satzCL}[2]{\ensuremath{(#1|#2)}}

\newlength{\abstand}
\newcommand{\symbolForInfSkept}{\ensuremath{\mathit{sk}}}

\newcommand{\symbolInd}{\ensuremath{\mathit{\mathit{O}}}}

\newcommand{\kbableitModMinName}[3]{\ensuremath{\nmableit^{\!\!\!\!#2,#3}_{\!#1}}}

\newcommand{\indinf}[3]{\ensuremath{\kbableitModMinName{\R}{\symbolForInfSkept}{\symbolInd}}}

\newcommand{\NFallKonditionale}[1]{\ensuremath{\mathit{NFC}(#1)}}

\newcommand{\closureMatrixName}[1]{\ensuremath{\mathit{CM}}} %

\newcommand{\listL}[1]{\ensuremath{\lbrack\mathcal{L}_{\Sigma}\rbrack}} %
\newcommand{\listNFC}[1]{\ensuremath{\lbrack\NFallKonditionale{\Sigma}\rbrack}} %

\makeatletter
\newcommand{\preceqdotLOKAL}{\mathrel{\mathpalette\LOKALpr@ceqd@t\relax}}
\newcommand{\LOKALpr@ceqd@t}[2]{%
  \begingroup
  \sbox\z@{$#1\prec$}\sbox\tw@{$#1\preccurlyeq$}%
  \dimen@=0.5\dimexpr\ht\tw@-\ht\z@\relax
  {\preccurlyeq}%
  \mkern-4mu
  \raisebox{\dimen@}{$\m@th#1\cdot$}%
  \endgroup
}
\makeatother

\makeatletter
\newcommand{\preceqprecdotLOKAL}{\mathrel{\mathpalette\LOKALprpr@ceqd@t\relax}}
\newcommand{\LOKALprpr@ceqd@t}[2]{%
  \begingroup
  \sbox\z@{$#1\prec$}\sbox\tw@{$#1\preccurlyeq$}%
  \dimen@=0.5\dimexpr\ht\tw@-\ht\z@\relax
  {\preccurlyeq}%
  \mkern-9mu
  {\prec}%
  \mkern-4mu
  \raisebox{\dimen@}{$\m@th#1\cdot$}%
  \endgroup
}
\makeatother

\newcommand{\displayTransformationRule}[3]%
{#1 &  #2 & #3}

\newcommand{\displayTRrule}[3]%
{#1 &  #2 & #3}

{
     
    \begin{enumerate}
}
{
    \end{enumerate}
}

{

    \begin{enumerate}
}
{
    \end{enumerate}
}

\newcommand\satzCLab[4]%
{\ifthenelse{\equal{#1}{no}}{\ensuremath{\satzCL{\{#3\}}{\{#4\}}}}{\ensuremath{\hbox{\ensuremath{#1_{#2}\!\!:\!\,\satzCL{\{#3\}}{\{#4\}}}}}}}

\newcommand{\nfcab}[2][no]{%
    \IfEqCase{#2}{%
        {01}{\ensuremath{\satzCLab{#1}{#2}{3}{3,2}}}%
        {02}{\ensuremath{\satzCLab{#1}{#2}{2}{3,2}}}%
        {03}{\ensuremath{\satzCLab{#1}{#2}{3}{3,0}}}%
        {04}{\ensuremath{\satzCLab{#1}{#2}{0}{3,0}}}%
        {05}{\ensuremath{\satzCLab{#1}{#2}{2}{2,1}}}%
        {06}{\ensuremath{\satzCLab{#1}{#2}{2}{2,0}}}%
        {07}{\ensuremath{\satzCLab{#1}{#2}{0}{2,0}}}%
        {08}{\ensuremath{\satzCLab{#1}{#2}{3}{3,2,1}}}%
        {09}{\ensuremath{\satzCLab{#1}{#2}{2}{3,2,1}}}%
        {12}{\ensuremath{\cb{\satzCLab{#1}{#2}{3}{3,2,0}}}}%
        {13}{\ensuremath{\cb{\satzCLab{#1}{#2}{2}{3,2,0}}}}%
        {14}{\ensuremath{\cb{\satzCLab{#1}{#2}{0}{3,2,0}}}}%
        {18}{\ensuremath{\cb{\satzCLab{#1}{#2}{2}{2,1,0}}}}%
        {19}{\ensuremath{\cb{\satzCLab{#1}{#2}{0}{2,1,0}}}}%
        {10}{\ensuremath{\cb{\satzCLab{#1}{#2}{3,2}{3,2,1}}}}%
        {11}{\ensuremath{\cb{\satzCLab{#1}{#2}{2,1}{3,2,1}}}}%
        {15}{\ensuremath{\cb{\satzCLab{#1}{#2}{3,2}{3,2,0}}}}%
        {16}{\ensuremath{\cb{\satzCLab{#1}{#2}{3,0}{3,2,0}}}}%
        {17}{\ensuremath{\cb{\satzCLab{#1}{#2}{2,0}{3,2,0}}}}%
        {20}{\ensuremath{\satzCLab{#1}{#2}{2,1}{2,1,0}}}%
        {21}{\ensuremath{\satzCLab{#1}{#2}{2,0}{2,1,0}}}%
        {22}{\ensuremath{\satzCLab{#1}{#2}{3}{3,2,1,0}}}%
        {23}{\ensuremath{\satzCLab{#1}{#2}{2}{3,2,1,0}}}%
        {24}{\ensuremath{\satzCLab{#1}{#2}{0}{3,2,1,0}}}%
        {25}{\ensuremath{\satzCLab{#1}{#2}{3,2}{3,2,1,0}}}%
        {26}{\ensuremath{\satzCLab{#1}{#2}{3,0}{3,2,1,0}}}%
        {27}{\ensuremath{\satzCLab{#1}{#2}{2,1}{3,2,1,0}}}%
        {28}{\ensuremath{\satzCLab{#1}{#2}{2,0}{3,2,1,0}}}%
        {29}{\ensuremath{\satzCLab{#1}{#2}{3,2,1}{3,2,1,0}}}%
        {30}{\ensuremath{\satzCLab{#1}{#2}{3,2,0}{3,2,1,0}}}%
        {31}{\ensuremath{\satzCLab{#1}{#2}{2,1,0}{3,2,1,0}}}%
        {32}{\ensuremath{\satzCLab{#1}{#2}{3}{3,1}}}%
        {33}{\ensuremath{\satzCLab{#1}{#2}{1}{3,1}}}%
        {34}{\ensuremath{\satzCLab{#1}{#2}{1}{2,1}}}%
        {35}{\ensuremath{\satzCLab{#1}{#2}{1}{1,0}}}%
        {36}{\ensuremath{\satzCLab{#1}{#2}{0}{1,0}}}%
        {37}{\ensuremath{\satzCLab{#1}{#2}{1}{3,2,1}}}%
        {39}{\ensuremath{\cb{\satzCLab{#1}{#2}{3}{3,1,0}}}}%
        {40}{\ensuremath{\cb{\satzCLab{#1}{#2}{1}{3,1,0}}}}%
        {41}{\ensuremath{\cb{\satzCLab{#1}{#2}{0}{3,1,0}}}}%
        {45}{\ensuremath{\cb{\satzCLab{#1}{#2}{1}{2,1,0}}}}%
        {38}{\ensuremath{\cb{\satzCLab{#1}{#2}{3,1}{3,2,1}}}}%
        {42}{\ensuremath{\cb{\satzCLab{#1}{#2}{3,1}{3,1,0}}}}%
        {43}{\ensuremath{\cb{\satzCLab{#1}{#2}{3,0}{3,1,0}}}}%
        {44}{\ensuremath{\cb{\satzCLab{#1}{#2}{1,0}{3,1,0}}}}%
        {46}{\ensuremath{\satzCLab{#1}{#2}{1,0}{2,1,0}}}%
        {47}{\ensuremath{\satzCLab{#1}{#2}{1}{3,2,1,0}}}%
        {48}{\ensuremath{\satzCLab{#1}{#2}{3,1}{3,2,1,0}}}%
        {49}{\ensuremath{\satzCLab{#1}{#2}{1,0}{3,2,1,0}}}%
        {50}{\ensuremath{\satzCLab{#1}{#2}{3,1,0}{3,2,1,0}}}%
    }[\PackageError{nfcab}{Undefined option to nfcab: #1}{}]%
}%

\newcommand{\nfcabChar}[2][no]{%
    \IfEqCase{#2}{%
        {01}{\ensuremath{\satzCLab{#1}{#2}{ab}{ab,a\ol{b}}}}%
        {02}{\ensuremath{\satzCLab{#1}{#2}{a\ol{b}}{ab,a\ol{b}}}}%
        {03}{\ensuremath{\satzCLab{#1}{#2}{ab}{ab,\ol{a}\ol{b}}}}%
        {04}{\ensuremath{\satzCLab{#1}{#2}{\ol{a}\ol{b}}{ab,\ol{a}\ol{b}}}}%
        {05}{\ensuremath{\satzCLab{#1}{#2}{a\ol{b}}{a\ol{b},\ol{a}b}}}%
        {06}{\ensuremath{\satzCLab{#1}{#2}{a\ol{b}}{a\ol{b},\ol{a}\ol{b}}}}%
        {07}{\ensuremath{\satzCLab{#1}{#2}{\ol{a}\ol{b}}{a\ol{b},\ol{a}\ol{b}}}}%
        {08}{\ensuremath{\satzCLab{#1}{#2}{ab}{ab,a\ol{b},\ol{a}b}}}%
        {09}{\ensuremath{\satzCLab{#1}{#2}{a\ol{b}}{ab,a\ol{b},\ol{a}b}}}%
        {12}{\ensuremath{\satzCLab{#1}{#2}{ab}{ab,a\ol{b},\ol{a}\ol{b}}}}%
        {13}{\ensuremath{\satzCLab{#1}{#2}{a\ol{b}}{ab,a\ol{b},\ol{a}\ol{b}}}}%
        {14}{\ensuremath{\satzCLab{#1}{#2}{\ol{a}\ol{b}}{ab,a\ol{b},\ol{a}\ol{b}}}}%
        {18}{\ensuremath{\satzCLab{#1}{#2}{a\ol{b}}{a\ol{b},\ol{a}b,\ol{a}\ol{b}}}}%
        {19}{\ensuremath{\satzCLab{#1}{#2}{\ol{a}\ol{b}}{a\ol{b},\ol{a}b,\ol{a}\ol{b}}}}%
        {10}{\ensuremath{\satzCLab{#1}{#2}{ab,a\ol{b}}{ab,a\ol{b},\ol{a}b}}}%
        {11}{\ensuremath{\satzCLab{#1}{#2}{a\ol{b},\ol{a}b}{ab,a\ol{b},\ol{a}b}}}%
        {15}{\ensuremath{\satzCLab{#1}{#2}{ab,a\ol{b}}{ab,a\ol{b},\ol{a}\ol{b}}}}%
        {16}{\ensuremath{\satzCLab{#1}{#2}{ab,\ol{a}\ol{b}}{ab,a\ol{b},\ol{a}\ol{b}}}}%
        {17}{\ensuremath{\satzCLab{#1}{#2}{a\ol{b},\ol{a}\ol{b}}{ab,a\ol{b},\ol{a}\ol{b}}}}%
        {20}{\ensuremath{\satzCLab{#1}{#2}{a\ol{b},\ol{a}b}{a\ol{b},\ol{a}b,\ol{a}\ol{b}}}}%
        {21}{\ensuremath{\satzCLab{#1}{#2}{a\ol{b},\ol{a}\ol{b}}{a\ol{b},\ol{a}b,\ol{a}\ol{b}}}}%
        {22}{\ensuremath{\satzCLab{#1}{#2}{ab}{ab,a\ol{b},\ol{a}b,\ol{a}\ol{b}}}}%
        {23}{\ensuremath{\satzCLab{#1}{#2}{a\ol{b}}{ab,a\ol{b},\ol{a}b,\ol{a}\ol{b}}}}%
        {24}{\ensuremath{\satzCLab{#1}{#2}{\ol{a}\ol{b}}{ab,a\ol{b},\ol{a}b,\ol{a}\ol{b}}}}%
        {25}{\ensuremath{\satzCLab{#1}{#2}{ab,a\ol{b}}{ab,a\ol{b},\ol{a}b,\ol{a}\ol{b}}}}%
        {26}{\ensuremath{\satzCLab{#1}{#2}{ab,\ol{a}\ol{b}}{ab,a\ol{b},\ol{a}b,\ol{a}\ol{b}}}}%
        {27}{\ensuremath{\satzCLab{#1}{#2}{a\ol{b},\ol{a}b}{ab,a\ol{b},\ol{a}b,\ol{a}\ol{b}}}}%
        {28}{\ensuremath{\satzCLab{#1}{#2}{a\ol{b},\ol{a}\ol{b}}{ab,a\ol{b},\ol{a}b,\ol{a}\ol{b}}}}%
        {29}{\ensuremath{\satzCLab{#1}{#2}{ab,a\ol{b},\ol{a}b}{ab,a\ol{b},\ol{a}b,\ol{a}\ol{b}}}}%
        {30}{\ensuremath{\satzCLab{#1}{#2}{ab,a\ol{b},\ol{a}\ol{b}}{ab,a\ol{b},\ol{a}b,\ol{a}\ol{b}}}}%
        {31}{\ensuremath{\satzCLab{#1}{#2}{a\ol{b},\ol{a}b,\ol{a}\ol{b}}{ab,a\ol{b},\ol{a}b,\ol{a}\ol{b}}}}%
        {32}{\ensuremath{\satzCLab{#1}{#2}{ab}{ab,\ol{a}b}}}%
        {33}{\ensuremath{\satzCLab{#1}{#2}{\ol{a}b}{ab,\ol{a}b}}}%
        {34}{\ensuremath{\satzCLab{#1}{#2}{\ol{a}b}{a\ol{b},\ol{a}b}}}%
        {35}{\ensuremath{\satzCLab{#1}{#2}{\ol{a}b}{\ol{a}b,\ol{a}\ol{b}}}}%
        {36}{\ensuremath{\satzCLab{#1}{#2}{\ol{a}\ol{b}}{\ol{a}b,\ol{a}\ol{b}}}}%
        {37}{\ensuremath{\satzCLab{#1}{#2}{\ol{a}b}{ab,a\ol{b},\ol{a}b}}}%
        {39}{\ensuremath{\satzCLab{#1}{#2}{ab}{ab,\ol{a}b,\ol{a}\ol{b}}}}%
        {40}{\ensuremath{\satzCLab{#1}{#2}{\ol{a}b}{ab,\ol{a}b,\ol{a}\ol{b}}}}%
        {41}{\ensuremath{\satzCLab{#1}{#2}{\ol{a}\ol{b}}{ab,\ol{a}b,\ol{a}\ol{b}}}}%
        {45}{\ensuremath{\satzCLab{#1}{#2}{\ol{a}b}{a\ol{b},\ol{a}b,\ol{a}\ol{b}}}}%
        {38}{\ensuremath{\satzCLab{#1}{#2}{ab,\ol{a}b}{ab,a\ol{b},\ol{a}b}}}%
        {42}{\ensuremath{\satzCLab{#1}{#2}{ab,\ol{a}b}{ab,\ol{a}b,\ol{a}\ol{b}}}}%
        {43}{\ensuremath{\satzCLab{#1}{#2}{ab,\ol{a}\ol{b}}{ab,\ol{a}b,\ol{a}\ol{b}}}}%
        {44}{\ensuremath{\satzCLab{#1}{#2}{\ol{a}b,\ol{a}\ol{b}}{ab,\ol{a}b,\ol{a}\ol{b}}}}%
        {46}{\ensuremath{\satzCLab{#1}{#2}{\ol{a}b,\ol{a}\ol{b}}{a\ol{b},\ol{a}b,\ol{a}\ol{b}}}}%
        {47}{\ensuremath{\satzCLab{#1}{#2}{\ol{a}b}{ab,a\ol{b},\ol{a}b,\ol{a}\ol{b}}}}%
        {48}{\ensuremath{\satzCLab{#1}{#2}{ab,\ol{a}b}{ab,a\ol{b},\ol{a}b,\ol{a}\ol{b}}}}%
        {49}{\ensuremath{\satzCLab{#1}{#2}{\ol{a}b,\ol{a}\ol{b}}{ab,a\ol{b},\ol{a}b,\ol{a}\ol{b}}}}%
        {50}{\ensuremath{\satzCLab{#1}{#2}{ab,\ol{a}b,\ol{a}\ol{b}}{ab,a\ol{b},\ol{a}b,\ol{a}\ol{b}}}}%
    }[\PackageError{nfcab}{Undefined option to nfcab: #1}{}]%
}%

\newcommand{\nfcabNeu}[2][no]{%
    \IfEqCase{#2}{%
        {01}{\ensuremath{\satzCLab{#1}{#2}{3}{3,2}}}%
        {02}{\ensuremath{\satzCLab{#1}{#2}{3}{3,1}}}%
        {03}{\ensuremath{\satzCLab{#1}{#2}{2}{3,2}}}%
        {04}{\ensuremath{\satzCLab{#1}{#2}{1}{3,1}}}%
        {05}{\ensuremath{\satzCLab{#1}{#2}{3}{3,0}}}%
        {06}{\ensuremath{\satzCLab{#1}{#2}{0}{3,0}}}%
        {07}{\ensuremath{\satzCLab{#1}{#2}{2}{2,1}}}%
        {08}{\ensuremath{\satzCLab{#1}{#2}{1}{2,1}}}%
        {09}{\ensuremath{\satzCLab{#1}{#2}{2}{2,0}}}%
        {10}{\ensuremath{\satzCLab{#1}{#2}{1}{1,0}}}%
        {11}{\ensuremath{\satzCLab{#1}{#2}{0}{2,0}}}%
        {12}{\ensuremath{\satzCLab{#1}{#2}{0}{1,0}}}%
        {13}{\ensuremath{\satzCLab{#1}{#2}{3}{3,2,1}}}%
        {14}{\ensuremath{\satzCLab{#1}{#2}{2}{3,2,1}}}%
        {15}{\ensuremath{\satzCLab{#1}{#2}{1}{3,2,1}}}%
        {16}{\ensuremath{\satzCLab{#1}{#2}{3,2}{3,2,1}}}%
        {17}{\ensuremath{\satzCLab{#1}{#2}{3,1}{3,2,1}}}%
        {18}{\ensuremath{\satzCLab{#1}{#2}{2,1}{3,2,1}}}%
        {19}{\ensuremath{\satzCLab{#1}{#2}{3}{3,2,0}}}%
        {20}{\ensuremath{\satzCLab{#1}{#2}{3}{3,1,0}}}%
        {21}{\ensuremath{\satzCLab{#1}{#2}{2}{3,2,0}}}%
        {22}{\ensuremath{\satzCLab{#1}{#2}{1}{3,1,0}}}%
        {23}{\ensuremath{\satzCLab{#1}{#2}{0}{3,2,0}}}%
        {24}{\ensuremath{\satzCLab{#1}{#2}{0}{3,1,0}}}%
        {25}{\ensuremath{\satzCLab{#1}{#2}{3,2}{3,2,0}}}%
        {26}{\ensuremath{\satzCLab{#1}{#2}{3,1}{3,1,0}}}%
        {27}{\ensuremath{\satzCLab{#1}{#2}{3,0}{3,2,0}}}%
        {28}{\ensuremath{\satzCLab{#1}{#2}{3,0}{3,1,0}}}%
        {29}{\ensuremath{\satzCLab{#1}{#2}{2,0}{3,2,0}}}%
        {30}{\ensuremath{\satzCLab{#1}{#2}{1,0}{3,1,0}}}%
        {31}{\ensuremath{\satzCLab{#1}{#2}{2}{2,1,0}}}%
        {32}{\ensuremath{\satzCLab{#1}{#2}{1}{2,1,0}}}%
        {33}{\ensuremath{\satzCLab{#1}{#2}{0}{2,1,0}}}%
        {34}{\ensuremath{\satzCLab{#1}{#2}{2,1}{2,1,0}}}%
        {35}{\ensuremath{\satzCLab{#1}{#2}{2,0}{2,1,0}}}%
        {36}{\ensuremath{\satzCLab{#1}{#2}{1,0}{2,1,0}}}%
        {37}{\ensuremath{\satzCLab{#1}{#2}{3}{3,2,1,0}}}%
        {38}{\ensuremath{\satzCLab{#1}{#2}{2}{3,2,1,0}}}%
        {39}{\ensuremath{\satzCLab{#1}{#2}{1}{3,2,1,0}}}%
        {40}{\ensuremath{\satzCLab{#1}{#2}{0}{3,2,1,0}}}%
        {41}{\ensuremath{\satzCLab{#1}{#2}{3,2}{3,2,1,0}}}%
        {42}{\ensuremath{\satzCLab{#1}{#2}{3,1}{3,2,1,0}}}%
        {43}{\ensuremath{\satzCLab{#1}{#2}{3,0}{3,2,1,0}}}%
        {44}{\ensuremath{\satzCLab{#1}{#2}{2,1}{3,2,1,0}}}%
        {45}{\ensuremath{\satzCLab{#1}{#2}{2,0}{3,2,1,0}}}%
        {46}{\ensuremath{\satzCLab{#1}{#2}{1,0}{3,2,1,0}}}%
        {47}{\ensuremath{\satzCLab{#1}{#2}{3,2,1}{3,2,1,0}}}%
        {48}{\ensuremath{\satzCLab{#1}{#2}{3,2,0}{3,2,1,0}}}%
        {49}{\ensuremath{\satzCLab{#1}{#2}{3,1,0}{3,2,1,0}}}%
        {50}{\ensuremath{\satzCLab{#1}{#2}{2,1,0}{3,2,1,0}}}%
    }[\PackageError{nfcab}{Undefined option to nfcab: #1}{}]%
}%

\newcommand{\nfcabCharNeu}[2][no]{%
    \IfEqCase{#2}{%
        {01}{\ensuremath{\satzCLab{#1}{#2}{ab}{ab,a\ol{b}}}}%
        {02}{\ensuremath{\satzCLab{#1}{#2}{ab}{ab,\ol{a}b}}}%
        {03}{\ensuremath{\satzCLab{#1}{#2}{a\ol{b}}{ab,a\ol{b}}}}%
        {04}{\ensuremath{\satzCLab{#1}{#2}{\ol{a}b}{ab,\ol{a}b}}}%
        {05}{\ensuremath{\satzCLab{#1}{#2}{ab}{ab,\ol{a}\ol{b}}}}%
        {06}{\ensuremath{\satzCLab{#1}{#2}{\ol{a}\ol{b}}{ab,\ol{a}\ol{b}}}}%
        {07}{\ensuremath{\satzCLab{#1}{#2}{a\ol{b}}{a\ol{b},\ol{a}b}}}%
        {08}{\ensuremath{\satzCLab{#1}{#2}{\ol{a}b}{a\ol{b},\ol{a}b}}}%
        {09}{\ensuremath{\satzCLab{#1}{#2}{a\ol{b}}{a\ol{b},\ol{a}\ol{b}}}}%
        {10}{\ensuremath{\satzCLab{#1}{#2}{\ol{a}b}{\ol{a}b,\ol{a}\ol{b}}}}%
        {11}{\ensuremath{\satzCLab{#1}{#2}{\ol{a}\ol{b}}{a\ol{b},\ol{a}\ol{b}}}}%
        {12}{\ensuremath{\satzCLab{#1}{#2}{\ol{a}\ol{b}}{\ol{a}b,\ol{a}\ol{b}}}}%
        {13}{\ensuremath{\satzCLab{#1}{#2}{ab}{ab,a\ol{b},\ol{a}b}}}%
        {14}{\ensuremath{\satzCLab{#1}{#2}{a\ol{b}}{ab,a\ol{b},\ol{a}b}}}%
        {15}{\ensuremath{\satzCLab{#1}{#2}{\ol{a}b}{ab,a\ol{b},\ol{a}b}}}%
        {16}{\ensuremath{\satzCLab{#1}{#2}{ab,a\ol{b}}{ab,a\ol{b},\ol{a}b}}}%
        {17}{\ensuremath{\satzCLab{#1}{#2}{ab,\ol{a}b}{ab,a\ol{b},\ol{a}b}}}%
        {18}{\ensuremath{\satzCLab{#1}{#2}{a\ol{b},\ol{a}b}{ab,a\ol{b},\ol{a}b}}}%
        {19}{\ensuremath{\satzCLab{#1}{#2}{ab}{ab,a\ol{b},\ol{a}\ol{b}}}}%
        {20}{\ensuremath{\satzCLab{#1}{#2}{ab}{ab,\ol{a}b,\ol{a}\ol{b}}}}%
        {21}{\ensuremath{\satzCLab{#1}{#2}{a\ol{b}}{ab,a\ol{b},\ol{a}\ol{b}}}}%
        {22}{\ensuremath{\satzCLab{#1}{#2}{\ol{a}b}{ab,\ol{a}b,\ol{a}\ol{b}}}}%
        {23}{\ensuremath{\satzCLab{#1}{#2}{\ol{a}\ol{b}}{ab,a\ol{b},\ol{a}\ol{b}}}}%
        {24}{\ensuremath{\satzCLab{#1}{#2}{\ol{a}\ol{b}}{ab,\ol{a}b,\ol{a}\ol{b}}}}%
        {25}{\ensuremath{\satzCLab{#1}{#2}{ab,a\ol{b}}{ab,a\ol{b},\ol{a}\ol{b}}}}%
        {26}{\ensuremath{\satzCLab{#1}{#2}{ab,\ol{a}b}{ab,\ol{a}b,\ol{a}\ol{b}}}}%
        {27}{\ensuremath{\satzCLab{#1}{#2}{ab,\ol{a}\ol{b}}{ab,a\ol{b},\ol{a}\ol{b}}}}%
        {28}{\ensuremath{\satzCLab{#1}{#2}{ab,\ol{a}\ol{b}}{ab,\ol{a}b,\ol{a}\ol{b}}}}%
        {29}{\ensuremath{\satzCLab{#1}{#2}{a\ol{b},\ol{a}\ol{b}}{ab,a\ol{b},\ol{a}\ol{b}}}}%
        {30}{\ensuremath{\satzCLab{#1}{#2}{\ol{a}b,\ol{a}\ol{b}}{ab,\ol{a}b,\ol{a}\ol{b}}}}%
        {31}{\ensuremath{\satzCLab{#1}{#2}{a\ol{b}}{a\ol{b},\ol{a}b,\ol{a}\ol{b}}}}%
        {32}{\ensuremath{\satzCLab{#1}{#2}{\ol{a}b}{a\ol{b},\ol{a}b,\ol{a}\ol{b}}}}%
        {33}{\ensuremath{\satzCLab{#1}{#2}{\ol{a}\ol{b}}{a\ol{b},\ol{a}b,\ol{a}\ol{b}}}}%
        {34}{\ensuremath{\satzCLab{#1}{#2}{a\ol{b},\ol{a}b}{a\ol{b},\ol{a}b,\ol{a}\ol{b}}}}%
        {35}{\ensuremath{\satzCLab{#1}{#2}{a\ol{b},\ol{a}\ol{b}}{a\ol{b},\ol{a}b,\ol{a}\ol{b}}}}%
        {36}{\ensuremath{\satzCLab{#1}{#2}{\ol{a}b,\ol{a}\ol{b}}{a\ol{b},\ol{a}b,\ol{a}\ol{b}}}}%
        {37}{\ensuremath{\satzCLab{#1}{#2}{ab}{ab,a\ol{b},\ol{a}b,\ol{a}\ol{b}}}}%
        {38}{\ensuremath{\satzCLab{#1}{#2}{a\ol{b}}{ab,a\ol{b},\ol{a}b,\ol{a}\ol{b}}}}%
        {39}{\ensuremath{\satzCLab{#1}{#2}{\ol{a}b}{ab,a\ol{b},\ol{a}b,\ol{a}\ol{b}}}}%
        {40}{\ensuremath{\satzCLab{#1}{#2}{\ol{a}\ol{b}}{ab,a\ol{b},\ol{a}b,\ol{a}\ol{b}}}}%
        {41}{\ensuremath{\satzCLab{#1}{#2}{ab,a\ol{b}}{ab,a\ol{b},\ol{a}b,\ol{a}\ol{b}}}}%
        {42}{\ensuremath{\satzCLab{#1}{#2}{ab,\ol{a}b}{ab,a\ol{b},\ol{a}b,\ol{a}\ol{b}}}}%
        {43}{\ensuremath{\satzCLab{#1}{#2}{ab,\ol{a}\ol{b}}{ab,a\ol{b},\ol{a}b,\ol{a}\ol{b}}}}%
        {44}{\ensuremath{\satzCLab{#1}{#2}{a\ol{b},\ol{a}b}{ab,a\ol{b},\ol{a}b,\ol{a}\ol{b}}}}%
        {45}{\ensuremath{\satzCLab{#1}{#2}{a\ol{b},\ol{a}\ol{b}}{ab,a\ol{b},\ol{a}b,\ol{a}\ol{b}}}}%
        {46}{\ensuremath{\satzCLab{#1}{#2}{\ol{a}b,\ol{a}\ol{b}}{ab,a\ol{b},\ol{a}b,\ol{a}\ol{b}}}}%
        {47}{\ensuremath{\satzCLab{#1}{#2}{ab,a\ol{b},\ol{a}b}{ab,a\ol{b},\ol{a}b,\ol{a}\ol{b}}}}%
        {48}{\ensuremath{\satzCLab{#1}{#2}{ab,a\ol{b},\ol{a}\ol{b}}{ab,a\ol{b},\ol{a}b,\ol{a}\ol{b}}}}%
        {49}{\ensuremath{\satzCLab{#1}{#2}{ab,\ol{a}b,\ol{a}\ol{b}}{ab,a\ol{b},\ol{a}b,\ol{a}\ol{b}}}}%
        {50}{\ensuremath{\satzCLab{#1}{#2}{a\ol{b},\ol{a}b,\ol{a}\ol{b}}{ab,a\ol{b},\ol{a}b,\ol{a}\ol{b}}}}%
    }[\PackageError{nfcab}{Undefined option to nfcab: #1}{}]%
}%

\newcommand{\paarKlNr}[2]{\ensuremath{#1.#2}}

\newcommand{\nfcabNeuKlNr}[2][no]{%
  \IfEqCase{#2}{%
        {01}{\ensuremath{\satzCLab{#1}{\paarKlNr{01}{1}}{3}{3,2}}}%
        {02}{\ensuremath{\satzCLab{#1}{\paarKlNr{01}{2}}{3}{3,1}}}%
        {03}{\ensuremath{\satzCLab{#1}{\paarKlNr{02}{1}}{2}{3,2}}}%
        {04}{\ensuremath{\satzCLab{#1}{\paarKlNr{02}{2}}{1}{3,1}}}%
        {05}{\ensuremath{\satzCLab{#1}{\paarKlNr{03}{1}}{3}{3,0}}}%
        {06}{\ensuremath{\satzCLab{#1}{\paarKlNr{04}{1}}{0}{3,0}}}%
        {07}{\ensuremath{\satzCLab{#1}{\paarKlNr{05}{1}}{2}{2,1}}}%
        {08}{\ensuremath{\satzCLab{#1}{\paarKlNr{05}{2}}{1}{2,1}}}%
        {09}{\ensuremath{\satzCLab{#1}{\paarKlNr{06}{1}}{2}{2,0}}}%
        {10}{\ensuremath{\satzCLab{#1}{\paarKlNr{06}{2}}{1}{1,0}}}%
        {11}{\ensuremath{\satzCLab{#1}{\paarKlNr{07}{1}}{0}{2,0}}}%
        {12}{\ensuremath{\satzCLab{#1}{\paarKlNr{07}{2}}{0}{1,0}}}%
        {13}{\ensuremath{\satzCLab{#1}{\paarKlNr{08}{1}}{3}{3,2,1}}}%
        {14}{\ensuremath{\satzCLab{#1}{\paarKlNr{09}{1}}{2}{3,2,1}}}%
        {15}{\ensuremath{\satzCLab{#1}{\paarKlNr{09}{2}}{1}{3,2,1}}}%
        {16}{\ensuremath{\satzCLab{#1}{\paarKlNr{10}{1}}{3,2}{3,2,1}}}%
        {17}{\ensuremath{\satzCLab{#1}{\paarKlNr{10}{2}}{3,1}{3,2,1}}}%
        {18}{\ensuremath{\satzCLab{#1}{\paarKlNr{11}{1}}{2,1}{3,2,1}}}%
        {19}{\ensuremath{\satzCLab{#1}{\paarKlNr{12}{1}}{3}{3,2,0}}}%
        {20}{\ensuremath{\satzCLab{#1}{\paarKlNr{12}{2}}{3}{3,1,0}}}%
        {21}{\ensuremath{\satzCLab{#1}{\paarKlNr{13}{1}}{2}{3,2,0}}}%
        {22}{\ensuremath{\satzCLab{#1}{\paarKlNr{13}{1}}{1}{3,1,0}}}%
        {23}{\ensuremath{\satzCLab{#1}{\paarKlNr{14}{1}}{0}{3,2,0}}}%
        {24}{\ensuremath{\satzCLab{#1}{\paarKlNr{14}{1}}{0}{3,1,0}}}%
        {25}{\ensuremath{\satzCLab{#1}{\paarKlNr{15}{1}}{3,2}{3,2,0}}}%
        {26}{\ensuremath{\satzCLab{#1}{\paarKlNr{15}{2}}{3,1}{3,1,0}}}%
        {27}{\ensuremath{\satzCLab{#1}{\paarKlNr{16}{1}}{3,0}{3,2,0}}}%
        {28}{\ensuremath{\satzCLab{#1}{\paarKlNr{16}{2}}{3,0}{3,1,0}}}%
        {29}{\ensuremath{\satzCLab{#1}{\paarKlNr{17}{1}}{2,0}{3,2,0}}}%
        {30}{\ensuremath{\satzCLab{#1}{\paarKlNr{17}{2}}{1,0}{3,1,0}}}%
        {31}{\ensuremath{\satzCLab{#1}{\paarKlNr{18}{1}}{2}{2,1,0}}}%
        {32}{\ensuremath{\satzCLab{#1}{\paarKlNr{18}{2}}{1}{2,1,0}}}%
        {33}{\ensuremath{\satzCLab{#1}{\paarKlNr{19}{1}}{0}{2,1,0}}}%
        {34}{\ensuremath{\satzCLab{#1}{\paarKlNr{20}{1}}{2,1}{2,1,0}}}%
        {35}{\ensuremath{\satzCLab{#1}{\paarKlNr{21}{1}}{2,0}{2,1,0}}}%
        {36}{\ensuremath{\satzCLab{#1}{\paarKlNr{21}{2}}{1,0}{2,1,0}}}%
        {37}{\ensuremath{\satzCLab{#1}{\paarKlNr{22}{1}}{3}{3,2,1,0}}}%
        {38}{\ensuremath{\satzCLab{#1}{\paarKlNr{23}{1}}{2}{3,2,1,0}}}%
        {39}{\ensuremath{\satzCLab{#1}{\paarKlNr{23}{2}}{1}{3,2,1,0}}}%
        {40}{\ensuremath{\satzCLab{#1}{\paarKlNr{24}{1}}{0}{3,2,1,0}}}%
        {41}{\ensuremath{\satzCLab{#1}{\paarKlNr{25}{1}}{3,2}{3,2,1,0}}}%
        {42}{\ensuremath{\satzCLab{#1}{\paarKlNr{25}{2}}{3,1}{3,2,1,0}}}%
        {43}{\ensuremath{\satzCLab{#1}{\paarKlNr{26}{1}}{3,0}{3,2,1,0}}}%
        {44}{\ensuremath{\satzCLab{#1}{\paarKlNr{27}{1}}{2,1}{3,2,1,0}}}%
        {45}{\ensuremath{\satzCLab{#1}{\paarKlNr{28}{1}}{2,0}{3,2,1,0}}}%
        {46}{\ensuremath{\satzCLab{#1}{\paarKlNr{28}{2}}{1,0}{3,2,1,0}}}%
        {47}{\ensuremath{\satzCLab{#1}{\paarKlNr{29}{1}}{3,2,1}{3,2,1,0}}}%
        {48}{\ensuremath{\satzCLab{#1}{\paarKlNr{30}{1}}{3,2,0}{3,2,1,0}}}%
        {49}{\ensuremath{\satzCLab{#1}{\paarKlNr{30}{2}}{3,1,0}{3,2,1,0}}}%
        {50}{\ensuremath{\satzCLab{#1}{\paarKlNr{31}{1}}{2,1,0}{3,2,1,0}}}%
    }[\PackageError{nfcab}{Undefined option to nfcab: #1}{}]%
}%

\newcommand{\beweisDirektDetailliert}[1]{}   %

\newcommand{\kuerzerFuerLICS}[1]{}           %

\newlength{\absetzenkleinerlaenge}
\newcount\stunde\newcount\minute
\stunde=\time 
\divide\stunde by 60 \minute=\time
\newcommand{\uhrzeit}{\makebox{\number\stunde\,:\,\ifnum\minute<10 %
0\fi\number\minute}}
\makeatletter
\newcommand{\ifLatexThree}[2]{\@ifpackageloaded{xparse}{#1}{#2}}
\newcommand{\ifAMSmath}[2]{\@ifpackageloaded{amsmath}{#1}{#2}}
\newcommand{\ifMathSCR}[2]{\@ifpackageloaded{mathrsfs}{#1}{#2}}
\newcommand{\ifMathHyperREF}[2]{\@ifpackageloaded{hyperref}{#1}{#2}}
\makeatother

\ifLatexThree{%
	\NewDocumentCommand{\headword}{s o m}{\IfBooleanTF{#1}{#3}{\textbf{#3}}\IfNoValueTF{#2}{\index{#3}}{\index{#2}}}%
}{%
	\makeatletter
	\def\@headword#1{\textbf{#1}\index{#1}}%
	\def\@@headword#1{#1\index{#1}}%
	\def\headword#1{\@ifstar\@headword{#1}\@@headword{#1}}%
	\makeatother
}

\makeatletter
\@ifundefined{naturals}{}{}
\@ifundefined{ol}{\newcommand{\ol}[1]{\overline{#1}}}{}
\makeatother

\makeatletter
\newcommand{\textlabelmarker}[1]{%
	\protected@edef\@currentlabel{#1}%
	\phantomsection%
}
\newcommand{\textlabel}[2]{%
	\textlabelmarker{#1}%
	#1\label{#2}%
}
\makeatother

\ifx\nmableit\undefined
\makeatletter
\ifx\centernot\undefined
\newcommand*{\centernot}{%
	\mathpalette\@centernot
}
\fi
\def\@centernot#1#2{%
	\mathrel{%
		\rlap{%
			\settowidth\dimen@{$\m@th#1{#2}$}%
			\kern.5\dimen@
			\settowidth\dimen@{$\m@th#1=$}%
			\kern-.5\dimen@
			$\m@th#1\not$%
		}%
		{#2}%
	}%
}
\makeatother
\DeclareRobustCommand\nmableitSymb{\mathrel|\mkern-.5mu\joinrel\sim} %
\newcommand{\nmableit}{\ensuremath{\mbox{$\,\nmableitSymb\,$}}} %
\fi

\makeatletter
\ifMathHyperREF{%
	\newcommand{\seqref}[1]{\hyperref[{#1}]{\textup{\tagform@split{\getrefnumber{#1}}}}}%
}{%
	\newcommand{\seqref}[1]{\textup{\tagform@split{\getrefnumber{#1}}}}%
}
\newcommand\tagform@split[1]{%
	\begingroup
	\m@th\normalfont(\ignorespaces #1\unskip\@@italiccorr)%
	\endgroup
}
\makeatother

\makeatletter
\newcommand{\leqnomode}{\tagsleft@true\let\veqno\@@leqno}
\newcommand{\reqnomode}{\tagsleft@false\let\veqno\@@eqno}
\newcommand{\pushright}[1]{\ifmeasuring@#1\else\omit\hfill$\displaystyle#1$\fi\ignorespaces}
\newcommand{\pushleft}[1]{\ifmeasuring@#1\else\omit$\displaystyle#1$\hfill\fi\ignorespaces}
\newcommand{\specialcell}[1]{\ifmeasuring@#1\else\omit$\displaystyle#1$\ignorespaces\fi}

\makeatother

\DeclareMathOperator{\Cn}{Cn}

\ifMathSCR{%
}{%
}

\usepackage{environ,xparse,xkeyval,ifthen}
\newif\ifpostulatepresent\postulatepresentfalse

\ExplSyntaxOn
\NewEnviron{postulate}[2]%
{%
	\tl_gclear:N \g_tmpa_tl%
	\tl_gput_right:Nn \g_tmpa_tl { \ifpostulatepresent\\[\smallskipamount]\fi\global\postulatepresenttrue }%
	\tl_gput_right:Nn \g_tmpa_tl { \noindent\ignorespaces\ifthenelse{\equal{#1}{}}{}{(\textlabel{#1}{pstl:#1})} & }%
	\tl_gput_right:No \g_tmpa_tl { \BODY }%
	\tl_gput_right:Nn \g_tmpa_tl { & \hspace{1.5\medskipamount} #2  }%
	\group_insert_after:N \g_tmpa_tl%
}
\ExplSyntaxOff

\newcommand{\dotcup}{\mathbin{\dot\cup}}

\DeclareMathOperator{\Th}{Th}

\newcommand{\LpropSig}[1]{\ensuremath{\mathcal{L}_{#1}}}

\newcommand{\UnivSig}[1]{\ensuremath{\Omega_{#1}}}

\newcommand{\RevOp}{\ast}

\newcommand{\ChangeOp}{\circ}

\newcommand{\thistheoremname}{}
\newenvironment{namedpostulate}[1]
{\renewcommand{\thistheoremname}{Postulate #1}%
	\begin{genericpostulate}}
	{\end{genericpostulate}}
 
\newtheorem*{genericpostulate}{\thistheoremname}

\newcommand{\modelTransformationName}{model transformation}
\newcommand{\modelTransformationsName}{model transformations}

\newcommand{\ModelTransformationsName}{Model transformations}
\newcommand{\languageIndependentName}{language independent}

\newcommand{\languageIndependenceName}{language independence}

\newcommand{\LiP}{(Language Independent P)\xspace}

\newcommand{\citeay}[1]{\citeauthor{#1} \shortcite{#1}}    %

\newcommand{\cb}[1]{#1}
\newcommand{\jh}[1]{#1}

\title{Model Transformations for Ranking Functions and Total Preorders}

\author{%
Jonas Haldimann$^1$\and
Christoph Beierle$^1$ %
\affiliations
$^1$FernUniversität in Hagen, Germany\\
\emails
\{jonas.haldimann, christoph.beierle\}@fernuni-hagen.de
}

\begin{document}

\maketitle

\begin{abstract}
	In the field of knowledge representation, the considered epistemic states are often based on propositional interpretations, also called worlds.
	E.g., %
	epistemic states of agents can be modelled by ranking functions or total preorders on worlds.
However, there are usually different ways of how to describe a real world situation in a propositional language;
this
	can be seen as different points of view on the same situation.
	In this paper we introduce the concept of \emph{model transformations} to convert an epistemic state
        from one point of view to another point of view,
        yielding a novel notion of equivalence of epistemic states.
        We show how the well-known advantages of syntax-splitting, originally developed for belief sets and later extended to
        representation of epistemic states and to nonmonotonic reasoning, can be exploited for belief revision via model transformation by uncovering splittings not being present before.
        Furthermore, we characterize situations where belief change operators 
        commute with model transformations.

\end{abstract}

\section{Introduction}
\label{sec:introduction}

In the field of knowledge representation, the considered objects are often based on propositional logic.
A statement can be modelled as a logical formula directly; a conditional  \(\satzCL{B}{A}\) formalizes a defeasible rule ``If \(A\) then usually \(B\)'' for logical formulas \(A, B\).
Other representations are based on propositional interpretations, also called (possible) worlds.
Epistemic states of agents can be modelled, e.g., by a ranking function assigning a rank to each world, a total preorder on the set of worlds, or a belief set which can be represented by the set of its models.
Common to these approaches is that they assume an underlying (propositional) signature on which the formulas are based and which determines the set of propositional interpretations occurring in the epistemic states.
When choosing which part of a situation is described with which atomic sentence,
there are often different ways to model the same subject.

\begin{example}
	\label{ex:coin}
	Two programs P1 and P2 are running on a computer.
	Usually either both or none of the programs has access to the internet, depending on whether the computer is connected to  a network with an internet connection.
	But sometimes a weird firewall configuration causes the situation that one of the program has internet access but not the other program.
	\jh{We could model this situation with two signature variables \(a, b\) where \(a\) is true if program P1 can access the internet and \(b\) is true if program P2 can access the internet.
	Another way of modelling would be to introduce two variables \(c, d\) where \(c\) is true if P1 has internet access and \(d\) is true if a weird firewall configuration is in place that allows exactly one program to access the internet.
	While the two ways of choosing are different, the four interpretations of each signature correspond to the same four elementary events.
	For example, the situation where P1 has internet access but P2 not is modelled by \(a\overline{b}\) %
	and by \(cd\), %
	respectively.}
\end{example}

The different approaches to modelling in the example can be seen as different points of view on the same situation.
\jh{In this paper, we introduce the concept of \emph{\modelTransformationsName} that} allows transforming between these points of view by establishing a connection between the worlds induced by each signature.
As epistemic states that can be transformed into each other by \modelTransformationsName\ can be seen as different points of view on the same \jh{situation.} %

Epistemic states are often used in combination with operations realizing belief changes or inferences.
If an operator uses only the semantic side of an epistemic state based on worlds, then applying this operator and a \modelTransformationName\ is equivalent to applying the \modelTransformationName\ first and then the operator.
We formalize such operations as \emph{\languageIndependentName}.

\jh{
One important property of an epistemic state is if it allows for syntax splittings. \citeay{Parikh99} introduced the concept of syntax splittings %
to formulate the revision postulate (P) describing that only the relevant parts of the belief base should be changed by belief revision operators.
Later the notion of syntax splitting was extended to ranking functions and total preorders on worlds \cite{Kern-IsbernerBrewka17}. 
As syntax splittings depend on the language used, applying a \modelTransformationName\ to an epistemic state might yield a new or a finer syntax splitting.}
In this paper, we generalize the syntax splitting postulate (P) to also consider syntax splittings that can be obtained by  a \modelTransformationName. %

To summarize, the main contributions of this paper are
\begin{itemize}
	\item the introduction of \emph{\modelTransformationsName} as transformations between different points of view for ranking functions and total preorders,
	\item the introduction of \emph{\languageIndependenceName} as property for operators on epistemic states,
	\item a generalized syntax splitting postulate for belief sets that considers \modelTransformationsName. %
\end{itemize}

In Section~\ref{sec:background} we briefly recall the required background on conditional logic.
In Section~\ref{sec:signature_transformations} we introduce \modelTransformationsName, and in Section~\ref{sec:language_independence} we consider \languageIndependentName\ operations.
We investigate syntax splitting in combination with \modelTransformationsName\ in Section~\ref{sec:syntax_splitting} before concluding and discussing future work in Section~\ref{sec:conclusion}.

\section{Background: Logic, OCFs, and TPOs}
\label{sec:background}

A \emph{(propositional) signature} is a finite set \(\Sigma\) of identifiers;
we denote the propositional language over \(\Sigma\) by \(\LpropSig{\Sigma}\).
Usually, we denote elements of the signatures with lowercase letters \(a, b, c, \dots\) and formulas with uppercase letters \(A, B, C, \dots\).
We may denote %
\(A \wedge B\) by \(AB\) and %
\(\neg A\) by \(\ol{A}\) for brevity of notation.
The set of interpretations over %
\(\Sigma\) is denoted as \(\UnivSig{\Sigma}\).
Interpretations are also called \emph{worlds} and \UnivSig{\Sigma} is called the \emph{universe}.
An interpretation \(\omega \in \UnivSig{\Sigma}\) is a \emph{model} of a formula \(A \in \LpropSig{\Sigma}\) if \(A\) holds in \(\omega\), %
denoted as \(\omega \models A\).
The set of models of a formula over \(\Sigma\) is \(\Mod_\Sigma(A) = \{\omega \in \UnivSig{\Sigma} \mid \omega \models A\}\).
A formula \(A\) \emph{entails} a formula \(B\) if \(\Mod_\Sigma(A) \subseteq \Mod_\Sigma(B)\), denoted as \(A \models B\).

The deductive closure of a set \(S\)  of formulas is \(\Cn(S) = \{F \in \LpropSig{\Sigma} \mid S \models F\}\); for formulas \(A, B, \dots\) we abbreviate \(\Cn(\{A, B, \dots\})\) \jh{by} \(\Cn(A, B, C, \dots)\).
A \emph{belief set} \(K\) is a deductively closed set of formulas, i.e., \(\Cn(K) = K\).
The theory for a set of interpretations \(I \subseteq \Omega_{\Sigma}\) is %
\(\Th(I) = \{F \in \LpropSig{\Sigma} \mid \omega \models F \text{ for every } \omega \in I\}\).
For sets \(S, T\) of formulas we define \(S + T = \Cn(S \cup T)\).

A \emph{conditional} \(\satzCL{B}{A}\) connects two formulas \(A, B\) and represents the rule ``If \(A\) then usually \(B\)''.
The formula \(A\) is called the \emph{antecedent} and the formula \(B\) the \emph{consequent} of the conditonal.
A finite set of conditionals is called a \emph{conditional belief base}.
We use a three-valued semantics for conditionals in this paper \cite{deFinetti37orig}.
For a world \(\omega\) a conditional \(\satzCL{B}{A}\) is either \emph{verified} by \(\omega\) if \(\omega \models AB\), \emph{falsified} by \(\omega\) if \(\omega \models A\ol{B}\), or \emph{not applicable} to \(\omega\) if \(\omega \models \ol{A}\).

\jh{Two popular semantics for conditionals and conditional knowledge bases} are ranking functions and total preorders.

A \emph{ranking function} \cite{Spohn88}, also called \emph{ordinal conditional function} (OCF), is a function \(\kappa: \UnivSig{\Sigma} \rightarrow \mathbb{N}_0 \cup \{\infty\}\) such that \(\kappa^{-1}(0) \neq \emptyset\).
The intuition of an OCF is that the rank of a world is lower if the world is more plausible.
Ranking functions are extended to formulas by \(\kappa(A) = \min_{\omega \in \mathit{Mod}(A)} \kappa(\omega)\) with \(\min_{\emptyset} (\ldots) = \infty\).
An OCF \(\kappa\) models a conditional \(\satzCL{B}{A}\), denoted as \(\kappa \models \satzCL{B}{A}\) if \(\kappa(AB) < \kappa(A\ol{B})\), i.e., if the verification of the conditional is strictly more plausible than its falsification.
An OCF \(\kappa\) models a conditional belief set \(\Delta\), denoted as \(\kappa \models \Delta\) if \(\kappa \models r\) for every \(r \in \Delta\).
A \emph{total preorder} (TPO) is a total, reflexive, and transitive binary  relation.
The meaning of a total preorder \(\preceq\) on \(\UnivSig{\Sigma}\) as model for an epistemic state is that \(\omega_1\) is at least as plausible as \(\omega_2\) iff \(\omega_1 \preceq \omega_2\) for \(\omega_1, \omega_2 \in \UnivSig{\Sigma}\).
Total preorders on worlds are extended to formulas by \(A \preceq B\) if \(\min(\Mod_\Sigma(A), \preceq) \preceq \min(\Mod_\Sigma(B), \preceq)\).
A total preorder \(\preceq\) models a conditional \(\satzCL{B}{A}\), denoted as \(\mathord{\preceq} \models \satzCL{B}{A}\) if \(AB \prec A\ol{B}\), i.e., if the verification of the conditional is strictly more plausible than its falsification.
A total preorder \(\preceq\) models a conditional belief set \(\Delta\), denoted as \(\mathord{\preceq} \models \Delta\), if \(\mathord{\preceq} \models r\) for every \(r \in \Delta\).

Belief sets, OCFs, and TPOs can each be used to model the epistemic state of an agent.
In an evolving world, an agent needs to update her beliefs to account for new information.
The process of including new beliefs into the current epistemic state and resolving possible inconsistencies is called belief revision.
Such belief revisions can be formalized by a belief revision operator \(\RevOp\) mapping the epistemic state before the  revision and the incoming information to the new epistemic state; \(K \RevOp A\) denotes the result of revising epistemic state \(K\) with the information \(A\)
\jh{(e.g., \cite{AlchourronGardenforsMakinson85,DarwichePearl1997,Parikh99}).
General belief changes are denoted as \(K \ChangeOp A\).}

To draw inferences from conditional beliefs inductive inference operators can be used.
\emph{Inductive inference operators} \cite{KernIsbernerBeierleBrewka2020KR} formalize the inductive completion of a conditional belief base according to an inference method;
they are defined as
	a mapping \(C: \R \mapsto \nmableit_{\R}\) that maps each belief base to an inference relation such that direct inference  (DI) and trivial vacuity (TV) are fulfilled, \jh{i.e.,
if \(\satzCL{B}{A} \in \Delta\) implies \(A \nmableit_{\R} B\) and
if \(\Delta = \emptyset\) and \(A \nmableit_{\R} B\) imply \(A \models B\).}

\section{Model Transformations}
\label{sec:signature_transformations}

In this paper, we want to formalize \jh{changes %
as illustrated
in Example~\ref{ex:coin}.}
The two approaches to model the situation in the example resulted in different descriptions of the same \jh{real-world situations}.
For the general case, we define model transformations as bijections between two universes over possibly different signatures.

\begin{definition}[\modelTransformationName]
	Let \(\Sigma_1, \Sigma_2\) \jh{be signatures of the same size.}
	A \emph{\modelTransformationName} is a bijective mapping \(\phi: \Omega_{\Sigma_1} \rightarrow \Omega_{\Sigma_2}\).
\end{definition}

\ModelTransformationsName\ \(\phi\) can be lifted to OCFs and TPOs.

\begin{definition}[\modelTransformationsName\ for OCFs and TPOs]
	For \(\kappa\) over \(\Sigma_1\) we define \(\phi(\kappa) = \kappa'\) where \(\kappa'\) is an OCF  over \(\Sigma_2\) such that \(\kappa'(\omega) = \kappa(\phi^{-1}(\omega))\) for any \(\omega \in \LpropSig{\Sigma_2}\).
	For \(\preceq\) over \(\Sigma_1\) we define \(\phi(\preceq) = {\preceq'}\) where \(\preceq'\) is a TPO over \(\Sigma_2\) such that \(\omega \preceq' \omega^\ast\) iff \(\phi^{-1}(\omega) \preceq \phi^{-1}(\omega^\ast)\) for any \(\omega, \omega^\ast \in \LpropSig{\Sigma_2}\).
\end{definition}

\jh{This definition %
implies}
\(\kappa(\omega) = \kappa'(\phi(\omega))\) for every \(\omega \in \Omega_{\Sigma}\) and
\(\omega \preceq \omega^\ast\) iff \(\phi(\omega) \preceq' \phi(\omega^\ast)\) for \(\omega, \omega^\ast \in \Omega_{\Sigma}\).

Note that \modelTransformationsName\ go far beyond renamings of the underlying signature as in \cite{JH_BeierleHaldimann2022AMAI}.
While each bijection \(\sigma: \Sigma_1 \rightarrow \Sigma_2\) induces a \modelTransformationName\ \(\phi_\sigma: \Omega_{\Sigma_1} \rightarrow \Omega_{\Sigma_2}\) by \(\phi_\sigma(\omega) = \sigma(\omega)\), in general, \modelTransformationsName\ cannot be obtained from signature renamings.

\begin{example}
	\label{ex:sig_transformation}
	Consider the signatures \(\Sigma_{abc} = \{a, b, c\}\) and \(\Sigma_{xyz} = \{x, y, z\}\).
	The function \jh{\(\phi: \Omega_{\Sigma_{abc}} \rightarrow \Omega_{\Sigma_{xyz}}\), %
\begin{align*}
	abc &\mapsto \overline{x}\overline{y}\overline{z} & \overline{a}bc &\mapsto xy\overline{z} & 
	ab\overline{c} &\mapsto x\overline{y}z & \overline{a}b\overline{c} &\mapsto \overline{x}\overline{y}z \\
	a\overline{b}c &\mapsto xyz & \overline{a}\overline{b}c &\mapsto \overline{x}yz &
	a\overline{b}\overline{c} &\mapsto \overline{x}y\overline{z} & \overline{a}\overline{b}\overline{c} &\mapsto x\overline{y}\overline{z}
\end{align*}}
	is a \modelTransformationName.
	We have \(\phi(\kappa_{abc}) = \kappa_{xyz}\) where \(\kappa_{abc}\) and \(\kappa_{xyz}\) are the OCFs displayed in Figure~\ref{fig:ex_ocf}.
\end{example}

\begin{figure}
	\begin{subfigure}[t]{.49\linewidth}
		\centering
		\begin{tikzpicture}
			\tikzstyle{number} = [anchor=south west]
			\tikzstyle{world}=[rectangle, rounded corners=4, fill=blue!20,anchor=south west, inner sep=2]
			\tikzstyle{markedworld}=[rectangle, rounded corners=4, fill=red!20, draw=black, anchor=south west]
			\def\xdiff{0.9}
			\def\ydiff{0.5}
			
			\node at (-1*\xdiff, 0*\ydiff) [number] {0};
			\node at (-1*\xdiff, 1*\ydiff) [number] {1};
			\node at (-1*\xdiff, 2*\ydiff) [number] {2};
			\node at (-1*\xdiff, 3*\ydiff) [number] {3};
			\node at (-1*\xdiff, 4*\ydiff) [number] {4};
			\node at (-1*\xdiff, 5*\ydiff) [number] {5};

			\node at (0*\xdiff, 0*\ydiff) [world] {$\bar{a}bc$};
			\node at (0*\xdiff, 1*\ydiff) [world] {$a\bar{b}c$};
			\node at (1*\xdiff, 1*\ydiff) [world] {$a\bar{b} \bar{c}$};
			\node at (0*\xdiff, 2*\ydiff) [world] {$\bar{a}\bar{b}\bar{c} $};
			\node at (1*\xdiff, 2*\ydiff) [world] {$\bar{a}\bar{b}c$};
			\node at (1*\xdiff, 3*\ydiff) [world] {$abc$};
			\node at (0*\xdiff, 4*\ydiff) [world] {$ab\bar{c}$};
			\node at (1*\xdiff, 5*\ydiff) [world] {$\bar{a}b\bar{c} $};
		\end{tikzpicture}
		\caption{OCF \(\kappa_{abc}\) over \(\Sigma = \{a, b, c\}\) without non-trivial syntax splitting.}
		\label{fig:ex_ocf_abc}
	\end{subfigure}
	\hfill
	\begin{subfigure}[t]{.49\linewidth}
		\centering
		\begin{tikzpicture}
			\tikzstyle{number} = [anchor=south west]
			\tikzstyle{world}=[rectangle, rounded corners=4, fill=blue!20,anchor=south west, inner sep=2]
			\tikzstyle{markedworld}=[rectangle, rounded corners=4, fill=red!20, draw=black, anchor=south west]
			\def\xdiff{0.9}
			\def\ydiff{0.5}
			
			\node at (-1*\xdiff, 0*\ydiff) [number] {0};
			\node at (-1*\xdiff, 1*\ydiff) [number] {1};
			\node at (-1*\xdiff, 2*\ydiff) [number] {2};
			\node at (-1*\xdiff, 3*\ydiff) [number] {3};
			\node at (-1*\xdiff, 4*\ydiff) [number] {4};
			\node at (-1*\xdiff, 5*\ydiff) [number] {5};

			\node at (0*\xdiff, 0*\ydiff) [world] {$xy\bar{z}$};
			\node at (0*\xdiff, 1*\ydiff) [world] {$xyz$};
			\node at (1*\xdiff, 1*\ydiff) [world] {$\bar{x}y\bar{z}$};
			\node at (0*\xdiff, 2*\ydiff) [world] {$x\bar{y}\bar{z}$};
			\node at (1*\xdiff, 2*\ydiff) [world] {$\bar{x}yz$};
			\node at (1*\xdiff, 3*\ydiff) [world] {$\bar{x}\bar{y}\bar{z}$};
			\node at (0*\xdiff, 4*\ydiff) [world] {$x\bar{y}z$};
			\node at (1*\xdiff, 5*\ydiff) [world] {$\bar{x}\bar{y}z$};
		\end{tikzpicture}
		\caption{OCF function \(\kappa_{xyz}\) over \(\Sigma = \{x, y, z\}\) with syntax splitting \(\{x\} \dotcup \{y, z\}\).}
		\label{fig:ex_ocf_xyz}
	\end{subfigure}
	\caption{Ranking functions from Example~\ref{ex:sig_transformation}}
	\label{fig:ex_ocf}
\end{figure}
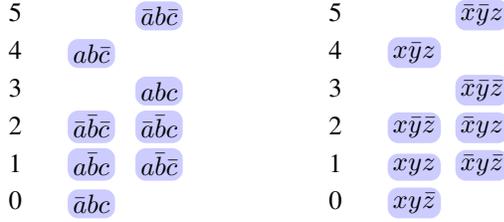

Applying \modelTransformationsName\ to formulas is more complex.
If we consider the syntactic structure of a formula, we cannot apply \modelTransformationsName\ directly.
But if we consider only formulas in canonical disjunctive normal form (CDNF) in clause form, i.e., formulas represented by the set of their models, we can apply \modelTransformationsName\ by  applying the transformations to each model of the formula, i.e., for \(A = \{\omega_1, \dots, \omega_n\}\) we define \(\phi(A) = \{\phi(\omega_1), \dots, \phi(\omega_n)\}\).
To work with formulas in this paper, from now on we assume that every formula is in CDNF.
\jh{\(\phi\) is lifted to conditionals by \(\phi(\satzCL{B}{A}) = \satzCL{\phi(B)}{\phi(A)}\) and to
belief sets by \(\phi(K) = \Th(\phi(\Mod(K)))\).}

\jh{\ModelTransformationsName\ are compatible with the models relation %
and with logical entailment.}

\begin{proposition}
	\label{prop:models_syntax_indep}
	Let \(\phi: \Omega_{\Sigma_1} \rightarrow \Omega_{\Sigma_2}\) be a \modelTransformationName.
	Let \(\omega \in \Omega_{\Sigma_1}\) and \(A, B \in \LpropSig{\Sigma_1}\).
 	\jh{Then \(\omega \models A\) iff \(\phi(\omega) \models \phi(A)\); and \(A \models B\) iff \(\phi(A) \models \phi(B)\).}
	
	\jh{Let \(\kappa\) be an OCF over \(\Sigma_1\) and \(\preceq\) be a TPO over \(\Sigma_1\).
	Let \(\satzCL{B}{A}\) be a conditional over \(\Sigma_1\).
Then
	\(\satzCL{B}{A} \models \kappa\) iff \(\phi(\satzCL{B}{A}) \models \phi(\kappa)\); and 
	\(\satzCL{B}{A} \models {\preceq}\) iff \(\phi(\satzCL{B}{A}) \models \phi(\preceq)\).}
\end{proposition}

Proposition~\ref{prop:models_syntax_indep} ensures that \(A \in K\) iff \(\phi(A) \in \phi(K)\) for any formula \(A\), belief set \(K\), and \modelTransformationName\ \(\phi\).

\section{Language Independent Operations}
\label{sec:language_independence}

While some operators depend on the valuation of signature variables in each world, many operators for belief change only consider %
worlds as atomic objects. %
With \modelTransformationsName\ we can formalize belief revision operators that do not depend on syntax at all. These operators are independent of the application of a \modelTransformationName.

\begin{definition}[\languageIndependentName\ belief change operators]
	A belief change operator \(\ChangeOp\) is called \emph{\languageIndependentName} if \(\phi(X) \circ \phi(Y) = \phi(X \circ Y)\) for each \modelTransformationName\ \(\phi\).
\end{definition}

Many belief change operators \jh{in the literature} are \languageIndependentName; they focus only on the semantic side of epistemic states and formulas.

\begin{proposition}
	\label{prop:lang_indep_change}
	The following belief change operators are \languageIndependentName:
	\begin{itemize}
		\item moderate, natural, and lexicographic contraction \cite{RamachandranNayakOrgun12} for TPOs
		\item natural revision \cite{Boutilier96} and simple lexicographic revision \cite{NayakPagnuccoPeppas2003} for TPOs 
		\item \jh{expansion \(+\)} \cite{AlchourronGardenforsMakinson85} and trivial update \cite{Parikh99} for belief sets.
	\end{itemize}
\jh{Dalal's revision operator \cite{Dalal88} for belief sets is not language independent.}
\end{proposition}
\begin{proof}[\jh{Proof sketch}]
This can be verified by considering the definitions of these operations.
\jh{The %
	moderate, natural, and lexicographic contraction as well as the natural and simple lexicographic revision} can be defined in a way that only considers the position of each world in the relation before the belief change and whether the world is a model of the input formula.
Expansion and trivial update can be also defined in a way that only considers if the worlds are  a model of the initial belief set and if they are a model of the input formula.
\jh{Dalal's revision is based on a TPO on worlds that compares the number of differently valued variables in different worlds.}
\end{proof}

\jh{
We can see that \languageIndependenceName\ is a property that occurs naturally in many revision operators,
but not every revision is \languageIndependentName. %
}

We can define \languageIndependenceName\ for inductive inference operators as well.

\begin{definition}[\languageIndependentName\ inference operators]
	An inductive inference operator \(C: \mathcal{R} \mapsto \nmableit_{\R}\) is called \emph{\languageIndependentName} if, for every \modelTransformationName\ \(\phi\),
it holds that 
\(A \nmableit_{\mathcal{R}} B \thickspace\) iff \(\thickspace\phi(A) \nmableit_{\phi(\mathcal{R})} \phi(B)\).
\end{definition}

There are many examples of \languageIndependentName\ inductive inference operators in the literature.

\begin{proposition}
	\label{prop:lang_indep_inf}
	P-entailment \cite{Adams1965}, system Z \cite{Pearl1990systemZTARK} and lexicographic inference \cite{Lehmann1995} %
	are \languageIndependentName\ inductive inference operators.
\end{proposition}

Similar to the belief change operators in Proposition~\ref{prop:lang_indep_change}, the inference operators in Proposition~\ref{prop:lang_indep_inf} \jh{are defined} in a way that only considers which conditionals in the belief base are verified and which are falsified by each world. %
Hence, they are \languageIndependentName.

\section{Transformations and Syntax Splitting}
\label{sec:syntax_splitting}

An important property of an epistemic state is whether it has a syntax splitting.
A syntax splitting is a partition of the signature describing that a belief set, a total preorder, or a rankinfunktion, respectively, consists of independent information on different parts of the signature partitioning \cite{Parikh99,Kern-IsbernerBrewka17}.
\jh{There are belief revision postulates describing that only the relevant part of the epistemic state %
must be revised.}
Respecting syntax splittings in belief revision leads to more intuitive revision operators %
and can also reduce the computational complexity of the belief revision by allowing to process several small parts of an epistemic state independently.

While syntax splittings are a highly desirable property they do depend on the underlying signature.
Interesting about \modelTransformationsName\ is that they can uncover new syntax splittings not being present before the transformation.

\begin{example}
	The OCF \(\kappa_{abc}\) from Example~\ref{ex:sig_transformation} does not have a non-trivial syntax splitting.
	The OCF \(\phi(\kappa_{abc}) = \kappa_{xyz}\) has the syntax splitting \(\{\{x\}, \{y, z\}\}\). 
\end{example}

To capture syntax splittings that exist only after application of a \modelTransformationName\ we introduce the following generalized notion of syntax splitting.

\begin{definition}[syntax splitting with respect to \modelTransformationsName] %
	\label{def:syn_split_mod_trans}
	Let \(\Sigma\) be a signature.
		Let \(X\) be a belief set, a TPO, or an OCF over \(\Sigma\).
		\jh{A \emph{syntax splitting for \(X\) with respect to \modelTransformationsName}} is a pair \((P, \phi)\) consisting of a partitioning \(P\) of \(\Sigma\) and a \modelTransformationName\ \(\phi: \Omega_{\Sigma} \rightarrow \Omega_{\Sigma}\) such that \(P\) is a syntax splitting for \(\phi(X)\).
	\end{definition}

	Note that the restriction to \modelTransformationsName\ from \(\Sigma\) to \(\Sigma\) does not limit the kind of partitions that occur in the syntax splittings.
	If we have a belief set, a TPO, or an OCF \(X\) and there is a \modelTransformationName\ \(\phi': \Omega_{\Sigma} \rightarrow \Omega_{\Sigma'}\) such that \(P\) is a syntax splitting for \(\phi'(X)\), \jh{then we can concatenate \(\phi'\) with the model transformation \(\phi_\sigma\) induced by a bijection \(\sigma: \Sigma' \rightarrow \Sigma\) on the signatures to obtain a \modelTransformationName\ \(\phi = \phi_\sigma \circ \phi'\) such that \(\phi: \Omega_{\Sigma} \rightarrow \Omega_{\Sigma}\) and \((\phi_\sigma(P), \phi)\) is a syntax splitting with respect to \modelTransformationsName\ for \(X\).}

	Syntax splitting with respect to \modelTransformationsName\ is a generalization of syntax splitting.

	\begin{proposition}
		\jh{If a belief set, a TPO, or an OCF \(X\) has a syntax splitting \(P\),  then \((P, \mathit{id})\) is a syntax splitting for \(X\) with respect to \modelTransformationsName\ with the identity \(\mathit{id}\).}
	\end{proposition}
	
	\begin{example}
		Consider %
		again Example~\ref{ex:sig_transformation}.
		Then \((\{\{a\}, \{b, c\}\}, \psi)\) is a syntax splitting with respect to \modelTransformationsName\ for the OCF \(\kappa_{abc}\) with \(\psi = \sigma \circ \phi\) and \(\sigma: \Sigma_{xyz} \rightarrow \Sigma_{abc}; x \mapsto a, y \mapsto b, z \mapsto c\).
	\end{example}

	For belief sets, i.e., deductively closed sets of propositional formulas, Parikh introduced the postulate (P) to describe \jh{that only the information about the relevant sub-signatures in the syntax splitting should be changed.}
		
	\begin{namedpostulate}{(P), see \cite{Parikh99}}
		Let \(K\) be a belief set and \(A\) a formula. If there is a syntax splitting \(\{\Sigma_1, \Sigma_2\}\) for \(K\), i.e., if there are \(C \in \LpropSig{\Sigma_1}, D \in \LpropSig{\Sigma_2}\) such that \(K = \Cn(C, D)\), and \(A \in \LpropSig{\Sigma_1}\), then
		\jh{
		\(
		K \RevOp A = (\Cn(C) \RevOp A) + D
		\)}.
	\end{namedpostulate}
	
	The postulate (P) not only ensures a more sensible outcome of belief revision operators, it is also useful for the computation of belief changes.
	Assume that we want to revise a belief set \(K = \Cn(C, D)\) \jh{with a syntax splitting \(\{\Sigma_1, \Sigma_2\}\) and \(C \in \LpropSig{\Sigma_1}, D \in \LpropSig{\Sigma_2}\) with a formula \(A \in \LpropSig{\Sigma_1}\).}
	If we use a revision operator that fulfils (P), we only have to calculate \(\Cn(C) \RevOp A\) and add \(D\) unchanged to obtain \(K \RevOp A\).
	
	\jh{We adapt} the syntax splitting postulate to the notion of syntax splitting with respect to \modelTransformationsName.
	
	\begin{namedpostulate}{\LiP}
		Let \(K\) be a belief set and \(A\) a formula. If \(K\) has a syntax splitting with respect to a \modelTransformationName\ \((\{\Sigma_1, \Sigma_2\}, \phi)\) and \(\phi(A) \in \LpropSig{\Sigma_1}\), then
		\(
			K \RevOp A = \phi^{-1}((\phi(K) \cap \LpropSig{\Sigma_1}) \RevOp \phi(A) + (\phi(K) \cap \LpropSig{\Sigma_2})).
		\)
	\end{namedpostulate}
	
	\jh{The intuition of \LiP is that  %
if the belief base has a syntax splitting with respect to \modelTransformationsName, then we should be able to conduct the revision from this point of view and respect the syntax splitting.}

		As the syntax splitting exists only in the transformed belief set, 
		we have to apply the \modelTransformationName\ of the synatx splitting to the belief set
		to separate the two parts.
		\jh{Using that \(\Cn(C, D) \cap \LpropSig{\Sigma_1} = \Cn(C)\) for \(C \in \LpropSig{\Sigma_1}, D \in \LpropSig{\Sigma_2}\) and \jh{\(\{\Sigma_1, \Sigma_2\}\) is a partition} of \(\Sigma\), the part of the belief set containing the information about \(\Sigma_i\) after the \modelTransformationName\ is %
		\(\phi(K) \cap \LpropSig{\Sigma_{i}}\)  for \(i \in \{1, 2\}\).}

	\jh{For revision operators that behave especially well with respect to \modelTransformationsName, \LiP can already be inferred from (P).}
	
	\begin{proposition}
		\jh{A \languageIndependentName\ revision \(\RevOp\) fulfils \LiP iff it fulfils (P).}
	\end{proposition}
	\begin{proof}
		To see that \LiP implies (P) consider the syntax splitting \((\{\Sigma_1, \Sigma_2\}, \mathit{id})\).

		For the other direction,
		let \(\RevOp\) be  a \languageIndependentName\ revision operator %
		that fulfils (P).
		Let \(K\) be a belief set such that \((\{\Sigma_1, \Sigma_2\}, \phi)\) is a syntax splitting with respect to \modelTransformationsName\ for \(K\) and let \(A\) be a formula such that \(\phi(A) \in \LpropSig{\Sigma_1}\).
		Then we have %
\(
	K \RevOp A = \phi^{-1}(\phi(K \RevOp A)) 
	= \phi^{-1}(\phi(K) \RevOp \phi(A)) 
	= \phi^{-1}((\phi(K) \cap \LpropSig{\Sigma_1}) \RevOp \phi(A) + (\phi(K) \cap \LpropSig{\Sigma_2}))
	\).
	\end{proof}
	
	\jh{
	\LiP can be applied in strictly more situations than (P), implying that \modelTransformationsName\ can uncover syntax splittings not being present before.
	\begin{proposition}
		\label{prop:LiP_more_applicable}
		There are belief sets that fulfil the prerequisites for \LiP but not for (P).
	\end{proposition}
	\begin{proof}
		Assume we have the belief set \(K = \Cn(a\ol{b} \vee \ol{a}b)\) over the signature \(\Sigma = \{a, b\}\) from Example~\ref{ex:coin}, i.e., we belive that exactly one of the two programs has internet access. Now we want to revise \(K\) with \(A = ab \vee \ol{a}\ol{b}\), i.e., we learn that we are actually in the usual situation that both or no programs have internet access.
		Even if we chose a revision operator fulfilling (P), we would have to consider the complete signature for this revision as \(K\) does not have a syntax splitting.
		However, \(K\) does have the syntax splitting with respect to \modelTransformationsName\ \((\{\{a\}, \{b\}\}, \phi)\) with \(\phi = \{ab \mapsto c\ol{d}, a\ol{b} \mapsto cd, \ol{a}b \mapsto \ol{c}d, \ol{a}\ol{b} \mapsto \ol{c}\ol{d}\}\).
		Divergent from Definition~\ref{def:syn_split_mod_trans} %
		 we use, just as in Example~\ref{ex:coin}, the different signature \(\{c,d\}\) for the transformed formulas to enhance readability,
		 where %
		 \(c\) is true if P1 has internet access and \(d\) is true if a weird firewall configuration is in place that allows exactly
		 one program to access the internet. 
In \(\phi(K)\) the information about these two things are independent.
		If we know that our revision operator fulfils \LiP, we can calculate \(K \RevOp A\) by calculating \((\phi(K) \cap \LpropSig{\Sigma_{1}}) \RevOp \phi(A) = \Cn(d) \RevOp \ol{d}\) on a smaller signature, combining it with \(\phi(K) \cap \LpropSig{\Sigma_{2}} = \top\), and transforming it back with \(\phi^{-1}\).
	\end{proof}
In the proof of Proposition~\ref{prop:LiP_more_applicable} we see how \LiP ensures that syntax splitting with respect to \modelTransformationsName\ is respected in a situation where (P) is not applicable.
Additionally, knowing that \(\RevOp\) fulfils \LiP allows to calculate the revision on only a part of the signature.
In applications with larger signatures where only a small part of the belief set is relevant for a revision utilizing syntax splittings with respect to \modelTransformationsName\ can be of advantage for the computation.
}
\section{Conclusion and Further work}
\label{sec:conclusion}

In this short paper we introduced the notion of \modelTransformationsName.
We outlined several applications of this notion, among them the definition of equivalence with respect to \modelTransformationsName, the definition of \languageIndependenceName\ as property of belief change and inference operators, and a generalized version of Parikh's postulate (P).

In our current work, we want to further investigate syntax splittings postulates in the context of \modelTransformationsName. Especially, we want to transfer the idea of \LiP to syntax splittings on OCFs and TPOs.
Another open question is if there are other properties of a belief base, OCF, or TPO besides syntax splitting that can be \jh{improved by} applying \modelTransformationsName.

\bibliographystyle{kr}
\newcommand{\verzeichnisBibtex}{\string~/BibTeXReferencesSVNlink}
\bibliography{%
	\verzeichnisBibtex/referencesCB,%
	\verzeichnisBibtex/referencesCBSekr,%
	\verzeichnisBibtex/literatur_sauerwald,%
	\verzeichnisBibtex/literatur_sauerwald_contr,%
	\verzeichnisBibtex/PapersKI,%
	\verzeichnisBibtex/ReferKI,%
	\verzeichnisBibtex/ReferKI2,%
	\verzeichnisBibtex/literatur_ce,%
	\verzeichnisBibtex/literatureJH%
}

\end{document}